\documentclass[11pt]{article}
\usepackage[utf8]{inputenc} 
\usepackage[T1]{fontenc}    
\usepackage[hidelinks]{hyperref}      
\usepackage{url}            
\usepackage{booktabs}       
\usepackage{nicefrac}       
\usepackage{microtype}      
\usepackage{xcolor}         
\usepackage[textsize = scriptsize]{todonotes}
\usepackage{amsmath,amssymb,amsfonts}
\usepackage{amsthm}
\usepackage{mathtools}
\usepackage{bm}
\usepackage{subcaption}
\usepackage{algorithm}
\usepackage{algpseudocode}  

\usepackage{bbm}
\usepackage{comment}
\usepackage{tikz}
\usepackage{enumerate}

\usepackage{tikz}
\usepackage{pgfplots}
\usepackage[export]{adjustbox}
\usetikzlibrary{spy}

\usepackage{geometry}

\newtheorem{theorem}{Theorem}
\newtheorem{lemma}[theorem]{Lemma} 
\newtheorem{proposition}[theorem]{Proposition}
\newtheorem{remark}[theorem]{Remark}

\theoremstyle{definition}

\usepackage{cleveref}

\usepackage{hyperref}
\usepackage{enumitem}

\newcommand{\assref}[1]{\hyperlink{#1}{(#1)}}

\newcounter{dummy} 

\newcommand{\R}{\mathbb{R}}

\renewcommand{\d}{\text{ d}}

\DeclareMathOperator*{\argmin}{arg\,min}
\DeclareMathOperator*{\argmax}{arg\,max}

\newcommand{\tT}{\mathrm{T}}

\definecolor{dgreen}{RGB}{0,100,0}

\begin{document}
\renewcommand{\thefootnote}{\fnsymbol{footnote}}

\title{Provable Mixed-Noise Learning with Flow-Matching}

\author{Paul Hagemann\footnotemark[1]\footnotemark[4] \and Robert Gruhlke\footnotemark[2]
\and Bernhard Stankewitz\footnotemark[3]
\and Claudia Schillings\footnotemark[2]
		\and Gabriele Steidl\footnotemark[1]
 }
 \date{\today}
 \maketitle
        \footnotetext[1]{Institute of Mathematics,
	   Technische Universität Berlin,
	   Stra{\ss}e des 17.\ Juni 136, 
	   10623 Berlin, Germany,
	   {\ttfamily\{hagemann, steidl\}@math.tu-berlin.de},
      \url{http://tu.berlin/imageanalysis}	
 }
 \footnotetext[2]{
 Institute of Mathematics,
 Freie Universität Berlin, 
  Arnimallee 6, 14195 Berlin, Germany,
 {\ttfamily\{r.gruhlke, c.schillings\}@fu-berlin.de},
 \url{https://www.mi.fu-berlin.de/math/groups/naspde/index.html}}
+
\footnotetext[3]{U Potsdam}

\footnotetext[4]{Bundesanstalt für Materialforschung und Prüfung (BAM),
	   {paul.hagemann@bam.de}
 }

\begin{abstract} \
We study Bayesian inverse problems with mixed noise, modeled as a combination of additive and multiplicative Gaussian components. 
While traditional inference methods often assume fixed or known noise characteristics, real-world applications, particularly in physics and chemistry, frequently involve noise with unknown and heterogeneous structure. 
Motivated by recent advances in flow-based generative modeling, we propose a novel inference framework based on conditional flow matching embedded within an Expectation-Maximization (EM) algorithm to jointly estimate posterior samplers and noise parameters.
To enable high-dimensional inference and improve scalability, we use simulation-free ODE-based flow matching as the generative model in the E-step of the EM algorithm. 
We find a more Bayesian interpretation of EM convergence conditions, and verify some slightly weaker ones on the mixed noise model. 
Our numerical results illustrate the effectiveness of combining EM inference with flow matching for mixed-noise Bayesian inverse problems.
\end{abstract}

\section{Introduction}

Bayesian inverse problems aim to infer an unknown variable $X \in \mathbb{R}^d$ from indirect and noisy observations $Y \in \mathbb{R}^n$. The relationship is typically modeled as
\[
Y_\theta = F(X) + \eta_\theta,
\]
where $F \colon \mathbb{R}^d \to \mathbb{R}^n$ is a known possibly nonlinear forward operator and $\eta_\theta$ represents random noise depending on parameters $\theta$. In many real-world settings, the noise is not known, but depends on auxiliary noise parameters $\theta$, which are in our case two scalars $(a,b)$.
We focus on so-called mixed noise conditional on $X = x$, which is a sum of additive and multiplicative Gaussian noise, i.e., 
$$
\eta_{\theta} \sim  \mathcal{N}\big(0, a^2+b^2 \text{diag} (F(x)^2) \big).
$$ 

As many physical phenomena (or approximations thereof) do not usually exhibit Gaussian errors, this often yields a more flexible modelling approach. It is present in optics \cite{herrero2021uncertainties, HGB2015} and also biology \cite{rocke2001model}. Furthermore, the mixed noise is a natural setting to extend to beyond additive noise that is still tractable for mathematical analysis. 

In \cite{Hagemann_2024}, an iterative procedure was introduced to jointly update normalizing flows and noise parameters, formulated within the Expectation-Maximization (EM) framework  consisting of an so called \textit{$E$-step} (expectation step) and a \textit{$M$-step} (maximization step) \cite{DLR1977}. This approach can be interpreted as a conditional extension of the DeepGEM algorithm \cite{gao2021deepgem}. However, the only convergence guarantee provided was the non-decreasing behavior of the marginal log-likelihood across iterations. Moreover, the method relies on discrete normalizing flows \cite{ruthotto2021introduction}, which have recently been surpassed in performance and scalability by ODE-based normalizing flows \cite{chen2018}, particularly those trained in a simulation-free manner via flow matching \cite{lipman2023flow, liu2023flow}. Flow matching has shown great success, in particular in image generation and regularization \cite{MGHS2025}, and is a variant of score-based diffusion \cite{song2021scorebased} which does not rely on a standard Gaussian latent.
This enables applications such as DNA sequence design, see~\cite{stark2024dirichlet}. Further, we can also make the latent distribution dependent on the data distribution, which yields optimal transport inspired flow matching approaches \cite{CHSW2024,tong2024improving}.

In this work, we build on and extend the line of research in \cite{Hagemann_2024}. Specifically, we
\begin{itemize}
    \item formulate an \emph{EM algorithm} where the E-step is realized via \emph{flow matching}, a modern generative modeling and sampling  technique based on neural ODEs;
    \item replace discrete flows with \emph{simulation-free, continuous-time flows} \cite{lipman2023flow} to enhance scalability and enable high-dimensional inference;
    \item establish \emph{consistency guarantees}: in the population limit of infinitely many measurements generated under a true noise parameter $\theta^*$, we investigate under which criteria the EM procedure converges, and give plausibility arguments for the mixed noise model. 
\end{itemize}

\subsection*{Notation}
We denote by $X \in \mathbb{R}^d$ the latent random variable, by $Y \in \mathbb{R}^n$ the observation and $\Theta\subset \mathbb R^d$ a nonempty, compact and convex parameter space. 
We denote by \( \| \cdot \| \) the Euclidean norm on \( \mathbb{R}^d \). For a function $\mathcal{Q}\colon \Theta\times\Theta\to \mathbb{R}$, we denote by $\nabla_1$ the gradient with respect to the first input.

Assuming existence (cp. \cite{stuart_2010}), we denote by $p_{Y_\theta}$ the density of the marginal likelihood of $Y$ given $\theta$, by $p_{X|Y_\theta=y}$ the corresponding posterior density and by $p_{X,Y_\theta}$ the joint density. The corresponding distributions are denoted using capitalization, e.g. $P_{X|Y_\theta = y}$ for the posterior distribution. We denote by \(\mathrm{pdf}(\mathbb{R}^d)\) the set of all Lebesgue densities on \(\mathbb{R}^d\), i.e.,
\[
\mathrm{pdf}(\mathbb{R}^d) := \left\{ q \colon \mathbb{R}^d \to [0,\infty) \;\middle|\; \int_{\mathbb{R}^d} q(x) \, \mathrm{d}x = 1 \right\},
\]
where the integral is taken with respect to the Lebesgue measure. For two probability measures $\mu$ and $\nu$ on a measurable space $(\Omega, \mathcal{B})$, the {Kullback--Leibler (KL) divergence} is defined by
\[
\mathrm{KL}(\mu\,\|\,\nu) \coloneqq \int_{\Omega} \log\left(\frac{\mathrm{d}\mu}{\mathrm{d}\nu}(x)\right) \, \mathrm{d}\mu(x),
\]
whenever $\mu$ is absolutely continuous with respect to $\nu$ and the integral is finite and set to $+\infty$ else. For measures $\mu$ and $\nu$ on $\mathbb{R}^d$ with Lebesque densities $f,g\in \mathrm{pdf}(\mathbb{R}^d)$ we may use abuse of notation and define $$\mathrm{KL}(f\,\|\,g) := \mathrm{KL}(\mu\,\|\,\nu).$$ 

For probability measures with finite first moments, the {Wasserstein-1 distance} is defined as
\[
\mathcal{W}_1(\mu, \nu) \coloneqq \inf_{\pi \in \Pi(\mu, \nu)} \int_{\Omega \times \Omega} \|x - y\| \, \mathrm{d}\pi(x,y),
\]
where $\Pi(\mu, \nu)$ denotes the set of all distributions on $\Omega\times\Omega$
 with marginals $\mu$ and $\nu$.
 \subsection*{Project Context}
The project can be seen in the context of the PhD project of the first author, supervised by the last author. In particular, it falls into the category of how one can solve inverse problems using conditional generative models. The main theme is that one can learn joint maps, i.e., families of generative models parameterized by the observation of the inverse problem. This is a very common theme in machine learning and for instance done in \cite{batzolis2021conditionalimagegenerationscorebased, ALKRK2019, Hagemann_2024}. This paper, in its theory also builds upon the classical Bayesian stability theory \cite{latz, sprungk2020local, stuart_2010} which was subsequently used to obtain pointwise estimates for conditional generators \cite{altekruger2023conditional} within this project. Therefore this paper revolves around the themes of the PhD project.

\subsection*{Outline}The manuscript at hand is organized as follows. We start giving the relevant background information on the general EM algorithm and flow matching in \Cref{sec:EM_Algo} and \Cref{sec:flowmatching} respectively. Next, we recall the convergence analysis of the so-called population EM operator in \Cref{sec:theory} and extend the theory in the follow up subsections  \Cref{sec:gradient_descent} and  \Cref{Sec:stepsize}. Then, we will verify the new framework for a class of mixed noise models in \Cref{sec:verification}. Finally, we illustrate the proposed method of mixed-noise learning with flow-matching in the numerical section for MNIST denoising in \Cref{sec:MNIST} and for a reconstruction task with underlying partial differential equation (PDE) in \Cref{sec:PDE}.

\section{Background}
Our goal is to recover the true noise parameters from observed data in a mathematically precise framework. To this end, we consider $N$ independent measurements $y_1, \dots, y_N$ drawn from a data-generating process corresponding to some unknown but fixed noise parameter $\theta^* \in \Theta$. Under standard regularity conditions, the maximum likelihood estimator is consistent, and in the limit of infinitely many observations, we may recover $\theta^*$ by solving
\begin{align}
\argmax_{\theta \in \Theta} \sum_{i=1}^N \log p_{Y_\theta}(y_i),
\end{align}
where $p_{Y_\theta}$ denotes the marginal likelihood of $Y$ under noise parameter $\theta$.

In practice, however, the marginal density $p_{Y_\theta}(y)$ is typically intractable, as it involves a high-dimensional integral of the form
\[
p_{Y_\theta}(y) = \int p_{Y_\theta|X = x}(y) \, p_X(x) \, \mathrm{d}x.
\]
A standard approach in variational inference and machine learning \cite{kingma2013, gao2021deepgem, Hagemann_2024} is to introduce a tractable surrogate by maximizing a lower bound on the log-likelihood, typically the so-called evidence lower bound (ELBO). This idea naturally leads to the EM framework, which alternates between estimating the latent posterior and optimizing the noise parameters.

\subsection{EM Algorithm}
\label{sec:EM_Algo}
Following the approach in \cite{Hagemann_2024}, we begin by observing that the marginal log-likelihood can be bounded below via Jensen's inequality. Let \( q\in \mathrm{pdf}(\mathbb{R}^d) \). Assuming \( q(x) > 0 \) wherever \( p_{X,Y_\theta}(x,y) > 0 \), we obtain
\begin{align*}
\log(p_{Y_\theta}(y)) 
&= \log\left( \int_{\mathbb{R}^d} p_{X,Y_\theta}(x,y) \, \mathrm{d}x \right) \\
&= \log\left( \int_{\mathbb{R}^d} \frac{p_{X,Y_\theta}(x,y)}{q(x)} q(x) \, \mathrm{d}x \right) \\
&\geq \int_{\mathbb{R}^d} \log\left( \frac{p_{X,Y_\theta}(x,y)}{q(x)} \right) q(x) \, \mathrm{d}x 
\;\eqqcolon\; \mathcal{F}(q, \theta \mid y),
\end{align*}
where \( \mathcal{F}(q, \theta \mid y) \) is known as the \emph{evidence lower bound (ELBO)}. This variational formulation motivates an iterative EM-type algorithm that alternates between approximating the posterior and updating the noise parameters.

More precisely, given current parameter estimate $\theta^{(k)}$, each iteration consists of the following steps:

\begin{align}
\text{E-step:} &\qquad q_i^{(k+1)} = \argmin_{q \in \mathrm{pdf}(\mathbb{R}^d)} \mathrm{KL}(q \,\|\, p_{X \mid Y_{\theta^{(k)}} = \,y_i}), \quad i = 1,\dots,N, \label{eq_estep} \\
\text{M-step:} &\qquad \theta^{(k+1)} = \argmax_{\theta \in \Theta} \sum_{i=1}^N \mathcal{F}(q_i^{(k+1)}, \theta \mid y_i), \label{eq_mstep}
\end{align}
where we have used the identity
\[
\log p_{Y_\theta}(y) = \mathrm{KL}(q \,\|\, p_{X \mid Y_\theta = y}) + \mathcal{F}(q, \theta \mid y),
\]
from \cite{bishop} which justifies the E-step as minimizing the KL divergence to the true posterior and the M-step as maximizing a lower bound on the marginal log-likelihood.

The concrete realization of this EM framework depends on how the posterior approximations $q_i$ and the M-step optimization are performed. In \cite{Hagemann_2024}, the M-step was addressed via a nested EM loop. Here, we instead use gradient-based optimization with automatic differentiation (autograd), which proves to be effective in practice, as demonstrated in Section~\ref{sec:numerics}.

For the E-step, classical approaches include MCMC-based posterior sampling. However, such methods are computationally demanding, especially when repeated for each observation. Alternatively, we adopt a conditional generative modeling perspective, where a single shared model is trained to approximate the posterior distribution for all observations. This amortized setup enables efficient inference across the dataset and allows parameter sharing across different $y_i$. While \cite{Hagemann_2024} relied on conditional normalizing flows in the spirit of \cite{gao2021deepgem}, we employ {conditional flow matching} as introduced in \cite{albergo2023building, lipman2023flow, liu2022rectified}, which provides improved stability and scalability in high-dimensional settings (which can be seen by improved image generation results as opposed to discrete normalizing flows).

\subsection{Flow Matching}
\label{sec:flowmatching}

We now detail the concrete realization of the E-step in our EM algorithm using conditional flow matching. 
The E-step aims to approximate the posterior distributions
\[
q_i^{(k+1)} = \argmin_{q \in \mathrm{pdf}(\mathbb{R}^d)} \mathrm{KL}(q, p_{X \mid Y_{\theta^{(k)}} = \, y_i}),
\]
for each observation \( y_i \), given current parameter estimates \( \theta^{(k)} \). In \cite{Hagemann_2024}, these $q_i$ were approximated using conditional normalizing flows, extending the DeepGEM framework \cite{gao2021deepgem} to the case of unknown noise parameters.

In this work, we instead propose to use \emph{conditional flow matching} \cite{albergo2023building,liu2022rectified, liu2023flow, lipman2023flow}, which constructs posterior samples by learning a velocity field that generates approximate transport maps from a reference distribution to the target posterior. Note that the term "conditional" here refers to the Bayesian conditioning on observations \( y \), which differs from the conditioning used in \cite{lipman2023flow}.

The central idea is to learn a measurement-conditional velocity field \( v_t(x, y) \) that transports a simple reference distribution \( P_Z \) to the posterior distribution \( P_{X \mid Y = y} \), for all values of \( y \) in the data distribution. To this end, we define a linear interpolation path in the joint space of observations and latent variables:
\[
(X_t, Y_t) := t \cdot (X, Y) + (1 - t) \cdot (Z, Y), \quad t \in [0,1],
\]
which starts at \( (Z, Y) \) at time \( t = 0 \) and ends at \( (X, Y) \) at time \( t = 1 \). This construction keeps the observation \( Y \) fixed along the path, i.e.,
\[
\partial_t (X_t, Y_t) = ( X - Z,0) \approx ( v_t(X_t, Y),0),
\]
so the velocity field only needs to act in the latent \( x \)-direction.

We denote by $P_{X_t,Y_t}=P_{X_t,Y}$ the distribution of $(X_t,Y_t)=(X_t,Y)$. 
Although this formulation above is formal, it can be made rigorous using continuity equations and transport maps (see, e.g., \cite{CHSW2024, SW2025}). The velocity field \( v_t(x, y) \) is parameterized by a neural network with parameters \( w \) and trained to minimize the loss
\[
L(w) = \mathbb{E}_{(x_t, y, t) \sim P_{X_t, Y} \otimes \,\mathcal{U}[0,1]} \left[ \left\| v_t(x_t, y) - (x_1 - x_0) \right\|^2 \right],
\]
which penalizes deviations from the idealized interpolation velocity. This objective also controls distributional discrepancies between the true posterior and the learned flow-based approximation, as shown in \cite{benton2024error}.

Once the velocity field is trained, samples from the approximate posterior \( p_{X \mid Y = y} \) are obtained by solving the associated neural ODE with fixed observation \( y \):
\[
\partial_t X_t = v_t(X_t, y), \quad X_0 = Z,
\]
where \( Z \sim P_Z \) is a latent sample from the base distribution.

\section{Convergence Analysis of the Population EM Operator}
\label{sec:theory}

We start by recalling the convergence results for the population EM algorithm from \cite{em_conv}. 
For $\theta\in\Theta$ consider a parametric family of distributions $P_{Y_{\theta}}$ with density $p_{Y_{\theta}}$.
We assume that the observed data are generated according to a distribution \( P_{Y_{\theta^*}} \) with density $ p_{Y_{\theta^*}} $
corresponding to a true but unknown parameter \( \theta^* \in \Theta \).
More precisely, we assume that \( \theta^* \) is the  unique maximizer of  the \emph{population log-likelihood}
\[
\theta^* = \argmax_{\theta \in \Theta} \, \mathbb{E}_{y \sim p_{Y_{\theta^*}}}[\log p_{Y_\theta}(y)].
\]
Define the population-level complete-data objective, the so-called \( Q \)-function, by
\[
\mathcal{Q}(\theta, \hat{\theta}) := \mathbb{E}_{y \sim p_{Y_{\theta^*}}} \, \mathbb{E}_{x \sim p_{X \mid Y_{\hat{\theta}} = y}} \left[ \log p_{X, Y_\theta}(x, y) \right].
\]
This function evaluates the expected complete-data log-likelihood under the posterior associated with \( \hat{\theta} \), and plays a central role in the EM framework.

It is known  \cite{em_conv} that the population maximum \( \theta^* \) satisfies the \emph{self-consistency condition}
\[
\theta^* = \argmax_{\theta \in \Theta} \, \mathcal{Q}(\theta, \theta^*).
\]

We define the population EM operator \( M \colon \Theta \to \Theta \) as
\[
M(\hat{\theta}) := \argmax_{\theta \in \Theta} \, \mathcal{Q}(\theta, \hat{\theta}),
\]
assuming uniqueness of the maximizer for $\hat{\theta}\in\Theta$,
so that the population EM iteration becomes
\[
\theta^{(k+1)} = M(\theta^{(k)}).
\]
A gradient-based variant is given by
\begin{equation}
\label{eq:gradient_descent}
\theta^{(k+1)} \coloneqq \theta^{(k)} + \tau \nabla_1 \mathcal{Q}(\theta, \theta^{(k)}) \big|_{\theta = \theta^{(k)}}, \quad k \in \mathbb{N}_0,
\end{equation}
where \( \tau > 0 \) is a step size. In the following, we assume that $\theta^*$ 
is in the interior of $\Theta$ and that $\varepsilon >0$
is chosen so small that  \( B_\varepsilon(\theta^*) := \{ \theta \in \Theta \mid \|\theta - \theta^*\| < \varepsilon \} \) is contained in $\Theta$.

\paragraph{Local Strong Concavity.}
We say that \( \mathcal{Q}(\cdot, \theta^*) \) is \emph{locally \( \lambda \)-strongly concave} for some \( \lambda > 0 \), if there exists \( \varepsilon > 0 \) such that
\[
\mathcal{Q}(\theta_1, \theta^*) - \mathcal{Q}(\theta_2, \theta^*) - \langle \nabla_1 \mathcal{Q}(\theta_2, \theta^*), \theta_1 - \theta_2 \rangle \leq -\frac{\lambda}{2} \|\theta_1 - \theta_2\|^2,
\]
for all \( \theta_1, \theta_2 \in B_\varepsilon(\theta^*) \). If \( \mathcal{Q} \) is twice continuously differentiable, this is equivalent to
\[
\nabla^2_\theta \mathcal{Q}(\theta, \theta^*) \leq -\lambda I_d, \quad \forall \theta \in B_\varepsilon(\theta^*).
\]

\paragraph{First-Order Stability (FOS).}
We say that 
\( \mathcal{Q}(\cdot, \theta^* ) \) 
satisfies the \emph{first-order stability condition} \( \mathrm{FOS}(\gamma) \) on \( B_\varepsilon(\theta^*) \), for some \( \gamma \geq 0 \), if
\[
\left\| \nabla_1 \mathcal{Q}(M(\theta), \theta^*) - \nabla_1 \mathcal{Q}(M(\theta), \theta) \right\| \leq \gamma \|\theta - \theta^*\|, \quad \forall \theta \in B_\varepsilon(\theta^*).
\]
Note that in general \( \nabla_1 \mathcal{Q}(M(\theta), \theta) \neq 0 \) due to the constrained maximization over \( \Theta \).

\begin{theorem}[Convergence of Population EM \cite{em_conv}]
\label{thm:main}
Let \( \varepsilon > 0 \), and suppose that \( \mathcal{Q}(\cdot, \theta^*) \) is locally \( \lambda \)-strongly concave, and that the \( \mathrm{FOS}(\gamma) \) condition holds on \( B_\varepsilon(\theta^*)\) with \( 0 \leq \gamma < \lambda \). Then the EM operator \( M \) is a contraction on \( B_\varepsilon(\theta^*) \), i.e.,
\[
\|M(\theta) - \theta^*\| \leq \frac{\gamma}{\lambda} \|\theta - \theta^*\|, \quad \forall \theta \in B_\varepsilon(\theta^*).
\]
In particular, for any \( \theta^{(0)} \in B_\varepsilon(\theta^*) \), the sequence satisfies linear convergence
\[
\|\theta^{(r)} - \theta^*\| \leq \left( \frac{\gamma}{\lambda} \right)^r \|\theta^{(0)} - \theta^*\|.
\]
\end{theorem}

\paragraph{Gradient Descent Variant.}
 The previous result assumes access to the exact maximizer \( M(\hat{\theta}) \). To relax this, a gradient descent version was proposed in \cite{em_conv}. 
 Define
\begin{itemize}
\item \( \mathcal{Q} \) is \emph{locally \( \gamma \)-gradient smooth} on $B_\varepsilon(\theta^*)$ if 
\[
\|\nabla_1 \mathcal{Q}({\theta}, \theta^*) - \nabla_1 \mathcal{Q}({\theta}, {\theta})\| \leq \gamma \|{\theta} - \theta^*\| \quad \forall {\theta} \in B_\varepsilon(\theta^*),
\]
\item
\( \mathcal{Q} \) is \emph{locally \( \mu \)-smooth} if
\[
\mathcal{Q}(\theta_1, \theta^*) - \mathcal{Q}(\theta_2, \theta^*) - \langle \nabla_1 \mathcal{Q}(\theta_2, \theta^*), \theta_1 - \theta_2 \rangle \geq -\frac{\mu}{2} \|\theta_1 - \theta_2\|^2
\quad 
\forall {\theta_1}, \theta_2 \in B_\varepsilon(\theta^*).
\]
\end{itemize}
Then we have the following theorem.

\begin{theorem}[Convergence of Gradient Descent EM \cite{em_conv}]
\label{thm:main_2}
Assume that \( \mathcal{Q} \) is locally \( \lambda \)-strongly concave, \( \mu \)-smooth, and \( \gamma \)-gradient smooth on \( B_\varepsilon(\theta^*)\) with \( 0 \leq \gamma < \lambda \leq \mu \). Let \( \theta^{(0)} \in B_\varepsilon(\theta^*) \), and consider the updates \eqref{eq:gradient_descent} with step size \( \tau = \frac{2}{\mu + \lambda} \). Then the iterates satisfy
\[
\|\theta^{(k)} - \theta^*\| \leq \left(1 - \frac{2(\lambda - \gamma)}{\mu + \lambda} \right)^k \|\theta^{(0)} - \theta^*\|.
\]
\end{theorem}

This result does no longer require the \( \mathrm{FOS}(\gamma) \) condition and relies instead on smoothness properties of the \( \mathcal{Q} \)-function.

\subsection{Sufficient Assumptions}
\label{Sec:sufficient_Ass}

In this subsection, we are interested in deriving sufficient conditions for two scenarios. First we focus on conditions that yield local strong concavity and first order stability for the case of population EM iteration. However, as it will turn out, the bound in the Wasserstein case will be too rough to give convergence, therefore we retreat to showing the FOS condition with high probability.

To this end, we formulate five assumptions:
\begin{itemize}[leftmargin=3.5em, labelsep=0.5em, align=left]
 \item[\textbf{(A1)}] \phantomsection\label{ass:A1}(\textit{Smoothness and Integrability of Log-Likelihood}) 
The map  $\theta \mapsto \log p_{X,Y_\theta}$is two times continuously differentiable in $B_\varepsilon(\theta^*)$ for all  $(x,y) \in  \R^{d \times n}$
and some $\varepsilon >0$ and there exist functions \( G, H: \R^{d \times n} \to \R\) such that for all  
$(x,y) \in  \R^{d \times n}$  and all
$\theta \in B_\varepsilon(\theta^*)$ it holds
\[
\left\| \nabla_{\theta} \log p_{X,Y_\theta}(x,y) \right\| \leq G(x,y), \quad 
\left\| \nabla_{\theta}^2 \log p_{X,Y_\theta}(x,y) \right\| \leq H(x,y) 
\]
with 
$$
\int_{\R^{d \times n}} G(x,y) \, p_{X|Y_\theta=y} \,  p_{Y_{\theta^*}} \d x \d y< C, \quad
\int_{\R^{d \times n}} H(x,y) \, p_{X|Y_\theta=y} \, p_{Y_{\theta^*}} \d x \d y< C.
$$
 \item[\textbf{(A2)}]\phantomsection\label{ass:A2}(\textit{Local identifiability of $\theta^\ast$}) 
The \emph{Fisher information matrix} 
$$\mathcal{I}(\theta^\ast) \coloneqq -\mathbb{E}_{x,y}[\nabla^2_{\theta} \log p_{X,Y_{\theta^\ast}}] $$ 
satisfies \( \mathcal{I}(\theta^*) \geq \lambda_{\ast} I_d \), for some $\lambda_{\ast} >0$ and denote the smallest constant $\mu^\ast$ with \( \mathcal{I}(\theta^*) \leq \mu^\ast I_d \) with $\mu^\ast > 0$.
 \item[\textbf{(A3)}]\phantomsection\label{ass:A3}
The map \( \theta \mapsto p_{X|Y_\theta=y} \) is locally 
 $L(y)$- Lipschitz in the Wasserstein-1 distance, i.e. for all 
 $(x,y) \in \R ^{d \times n}$
  $$
    W_1\left( P_{X|Y_{\theta_1}=y}, p_{X|Y_{\theta_2}=y} \right) \leq L(y)\| \theta_1 - \theta_2 \| \quad \forall \theta_1,\theta_2 \in B_\varepsilon(\theta^*) 
    $$  
 \item[\textbf{(A4)}]\label{ass:A4}The map \( x \mapsto \nabla_{\theta} \log p_{X,Y_\theta}(x,y) \) is $L_g(y,\theta)$-Lipschitz in \( x \), i.e.
    \[
    \left\| \nabla_{\theta} \log p_{X,Y_\theta}(x_1,y) - \nabla_{\theta} \log p_{X,Y_\theta}(x_2, y) \right\| 
    \leq  L_g  (y,\theta) \| x_1 - x_2 \| \quad \forall y \in \R^n, x_1,x_2 \in \mathbb{R}^d
    \]  
 \item[\textbf{(A5)}]\phantomsection\label{ass:A5}
  There exists $\bar \varepsilon >0$ such that for each $\varepsilon < \bar \varepsilon$, we have
    $$
    \gamma = \gamma(\varepsilon) = \sup\limits_{\theta\in B_\varepsilon(\theta^*)} \mathbb{E}_{y \sim p_{Y_{\theta^*}}}[L(y) L_g(\theta, y) ]  < \infty.
    $$  
 \item[\textbf{(A5F)}]\phantomsection\label{ass:A5F}
  There exists $\bar \varepsilon >0$ such that for each $\varepsilon < \bar \varepsilon$, we have
    $$
    \gamma = \gamma(\varepsilon) = \sup\limits_{\theta\in B_\varepsilon(\theta^*) }\mathbb{E}_{y \sim p_{Y_{\theta^*}}}[L(y) L_g({M}(\theta), y) ]  < \infty.
    $$    
 \item[\textbf{(A5-SG)}]\phantomsection\label{ass:A5SG}
$Y$ is sub-Gaussian with constant $K$ \cite{Vershynin_2018}, further there exist two continuous functions $h_1, h_2$, such that the constants $L(Y) \leq h_1(\| Y \|) $ and $L_g(\theta, Y) \leq h_2(\| Y \|)$ for each $\theta \in B_{\varepsilon}(\theta^*)$.
\end{itemize}

\subsection{FOS Framework}
In this section, we show that \hyperref[ass:A1]{(A1)}-\hyperref[ass:A4]{(A4)} and \hyperref[ass:A5F]{(A5F)}
are indeed sufficient for the assumptions of Theorem \ref{thm:main}.

\begin{proposition}[Strong Concavity]\label{lemma:concave}
    Let \hyperref[ass:A1]{(A1)} and \hyperref[ass:A2]{(A2)} be fulfilled. Then there exist $\varepsilon>0$ and $\lambda=\lambda(\varepsilon,\lambda_\ast)>0$ 
    such that 
    $\mathcal{Q}(\cdot,\theta^\ast)$ 
    is strong $\lambda$-concave on $B_\varepsilon(\theta^*)$.
\end{proposition}

\begin{proof}
    Under regularity assumptions of \hyperref[ass:A1]{(A1)} we can swap the following differentiation and integration, and applying subsequently \hyperref[ass:A2]{(A2)}, we get 
\begin{equation}
    \nabla_1^2 \mathcal{Q}(\theta^\ast,  \theta^\ast) = 
    \mathbb{E}_{y \sim p_{Y_{\theta^*}}} \mathbb{E}_{x \sim P_{X|Y_{\theta^*}=y}} \left[
    \nabla_{\theta}^2 \log p_{X,Y_{\theta^\ast}}(x,y)
    \right]  
    = - \mathcal{I}(\theta^\ast)
    \leq -\lambda_{\ast} I_d.\\ 
\end{equation}
By \hyperref[ass:A1]{(A1)} the map
$
\nabla_1^2 \mathcal{Q}(\cdot, \theta^\ast) 
$
is continuous on  $B_\varepsilon(\theta^*)$. Hence for $\varepsilon>0$ small enough, we have that
\begin{equation}
   \nabla_{\theta}^2\mathcal{Q}(\theta, \theta^\ast)  < -\lambda \mathrm{I}_d, \quad \forall \theta\in B_\varepsilon(\theta^*)
\end{equation}
for some $\lambda=\lambda(\varepsilon,\lambda_\ast)>0$, which yields the strong $\lambda$-concavity $\mathcal{Q}(\cdot,\theta^\ast)
$.
\end{proof}

\begin{proposition}[FOS Condition]\label{prop4}
    Let \hyperref[ass:A1]{(A1)}, \hyperref[ass:A3]{(A3)},\hyperref[ass:A4]{(A4)} and \hyperref[ass:A5F]{(A5F)} be fulfilled. Then $\mathcal{Q}$ satisfies a FOS$(\gamma)$ condition with $\gamma$ from \hyperref[ass:A5F]{(A5F)}.
\end{proposition}

\begin{proof}    
By the  Kantorovich-Rubenstein duality characterization of Wasserstein-1 metric we have 
$$
W_1(\mu,\nu) = \frac{1}{K}\sup\limits_{\phi \text{ is K-Lip}} \mathbb{E}_{x\sim \mu}[\phi(x)] - \mathbb{E}_{y\sim \nu}[\phi(y)].
$$
Let 
$\phi_{y,\theta}(x) := \nabla_1 \log p_{X,Y_\theta}(x,y)$ 
which is $L_g(y,\theta)$-Lipschitz by \hyperref[ass:A4]{(A4)}. 
Then, under the regularity assumptions of \hyperref[ass:A1]{(A1)}, we have
\begin{align*}
\nabla_1 \mathcal{Q}(\theta, \hat{\theta}) - \nabla_1 \mathcal{Q}(\theta, \theta^*) 
&= 
\mathbb{E}_{y \sim p_{Y_{\theta^*}} }
\Big[
\mathbb E_{x \sim P_{X|Y_{\hat \theta}=y}}
\left[
\nabla_{\theta} \log p_{X,Y_\theta}(x,y)
\right]
\\
&
\quad -
\mathbb E_{x \sim P_{X|Y_{\theta^*}=y}}
\left[
\nabla_{\theta} \log p_{X,Y_\theta}(x,y)
\right]
\Big]  \\
&= 
\mathbb{E}_{y \sim p_{Y_{\theta^*}}} \Big[
\mathbb{E}_{x \sim P_{X|Y_{\hat \theta}=y }}
\left[\phi_{y,\theta}(x)\right]
-
\mathbb E_{x \sim P_{X|Y_{\theta^*}=y}}
\left[\phi_{y,\theta}(x)
\right]
\Big]. 
\end{align*}
Hence for 
$\mu= \mu(y,\hat{\theta}) =  P_{X|Y_{\hat \theta}=y}$ 
and 
$\nu = \nu(y,\theta^\ast) = P_{X|Y_{\theta^*}=y} $, we have 
\begin{align*}
  &
  \|  \mathbb{E}_{x \sim P_{X|Y_{\hat{\theta}}=y} }
  \left[\phi_{y,\theta}(x)\right] 
    - \mathbb{E}_{x \sim P_{X|Y_{\theta^*}=y} }
    \left[\phi_{y,\theta}(x)\right]  
  \|
  \\
  & = 
  \sup_{\| v \| \le 1} 
  |\langle \mathbb{E}_{x \sim P_{X|Y_{\hat{\theta}}=y} }
  \left[ 
  \phi_{y,\theta}(x)
  \right] 
        - \mathbb{E}_{x \sim P_{X|Y_{\theta^*}=y } }
        \left[\phi_{y,\theta}(x)\right], v
  \rangle| 
  \\
  & = 
  \sup_{\| v \| \le 1} 
  | \mathbb{E}_{x \sim P_{X|Y_{\hat{\theta}}=y} } \left[ \langle \phi_{y,\theta}(x), v \rangle \right] 
  - \mathbb{E}_{x \sim P_{X|Y_{\theta^*}=y } } 
  \left[ 
  \langle \phi_{y,\theta}(x), v \rangle 
  \right] 
  |
  \\ 
  & \le L_g(y,\theta) W_1(\mu(y,\hat{\theta}), \nu(y,\theta^\ast)) \\
  &  \le L(y) L_g(\theta, y) \|\hat{\theta}- \theta^\ast\|.
\end{align*}
Integrating over \( y \sim p_{Y_{\theta^*}} \), we obtain
\[
\left\| \nabla_1 \mathcal{Q}(\theta, \hat{\theta}) - \nabla_1 \mathcal{Q}(\theta, \theta^*) \right\|
\leq 
\mathbb{E}_{y \sim p_{Y_{\theta^*}}}[L(y) L_g(\theta, y) ] \|\hat{\theta}- \theta^\ast\| .
\]
Inserting $\theta = M(\hat\theta)$, we obtain with \hyperref[ass:A5F]{(A5F)} that 
\begin{align*}
\left\| \nabla_1 \mathcal{Q}(M(\hat \theta), \hat{\theta}) - \nabla_1 \mathcal{Q}(M(\hat \theta), \theta^*) \right\|
&\leq \mathbb{E}_{y \sim p_{Y_{\theta^*}}}[L(y) L_g(M(\hat \theta), y) ] \|\hat{\theta}- \theta^\ast\| \\
&\leq \gamma \|\hat{\theta}- \theta^\ast\|.
\end{align*}
\end{proof}

\begin{remark}
The main line of argument also works if we replace Wasserstein by the total variation distance.  It is defined as
 \[\mathrm{TV}(P,Q) = \frac{1}{2} \int \vert p(x) - q(x) \vert dx\] for two probability measures $P,Q$ which admit densities $p$ and $q$. This also has a dual formulation, which is given by $$TV(P,Q) = \sup_{\Vert f \Vert_{\infty} \leq 1} \int f d(P-Q).$$

 Hence, we get slightly different assumptions: The score $\nabla_{\theta} \log p_{X,Y_{\theta}}$ now only needs to be element of $L^{\infty}(\mathbb{R}^{d})$ for all $y$, whereas the Lipschitz continuity bound is a bit harder to satisfy. 
\end{remark}

Note that in order to apply the convergence theorem, we must assume that \( 0 \leq \gamma < \lambda \). There are two main reasons why we prefer the following gradient-based framework over the classical EM convergence theory that relies on the FOS condition. 
First, the FOS condition involves the population EM operator \( M \), whose analytical form is often intractable and difficult to estimate numerically, especially in complex models. Second, Lemma~\ref{lem:conv_gs} provides a slightly improved convergence guarantee: instead of requiring \( \gamma < \lambda \), it suffices to assume \( \gamma < c \), where \( c \) is the strong convexity constant of the Q-function. This is a strictly weaker condition in non-degenerate cases. 
Therefore, we adopt the gradient-based framework, which is both more natural for our approximate M-step and yields more accessible convergence criteria.

As it will turn out, the integrability of $\gamma$ will not be true for our mixed noise model. However, luckily assumption \hyperref[ass:A5SG]{(A5SG)} will help us to show a similar statement with high probability. For this we look at the FOS condition in a probabilistic way instead, which means instead of bounding the expectation, we get a corresponding high probability statement. To rigorously formulate it, recall the definition of the $Q$-function and define $q$ via 
\[
\mathcal{Q}(\theta, \hat{\theta}) = \mathbb{E}_{y \sim p_{Y_{\theta^*}}} \, \mathbb{E}_{x \sim p_{X \mid Y_{\hat{\theta}} = y}} \left[ \log p_{X, Y_\theta}(x, y)\right] =   \mathbb{E}_{y \sim p_{Y_{\theta^*}}}[q(y, \hat\theta, \theta)].
\]

\begin{proposition}
Let \hyperref[ass:A1]{(A1)}, \hyperref[ass:A3]{(A3)},\hyperref[ass:A4]{(A4)} and \hyperref[ass:A5SG]{(A5SG)} be fulfilled.
For each $0 < \delta \leq 1$, there exists a measurable subset $\Omega_{\delta} \subset \Omega$ (where $(\Omega, \mathcal{F}, \mathbb{P})$ is the underlying probability space) 
and a constant $C_{\delta}>0$ such that
$$\Vert \nabla_{\theta}  q(Y(\omega),\theta^*, \theta)- \nabla_{\theta}   q(Y(\omega),\hat\theta, \theta) \Vert \leq C_{\delta} \Vert \theta^*-\hat\theta \Vert, \qquad \forall\omega\in\Omega_{\delta}, \theta \in B_{\varepsilon}(\theta^*)
$$
holds with probability $\mathbb{P}(\Omega_{\delta}) \geq 1-\delta$.
\end{proposition}
\begin{proof}
As in the proof of Proposition 2, we can bound, 
\begin{align*}
 \Vert \nabla_{\theta} q(y,\theta^*, \theta)- \nabla_{\theta}  q(y, \hat\theta, \theta) \Vert &  \le L(y) L_g(\theta, y) \|\hat{\theta}- \theta^\ast\|.
\end{align*}
The sub-Gaussianity from \hyperref[ass:A5SG]{(A5SG)} yields a bound 
$$\mathbb{P}(\Vert Y \Vert \geq t) \leq 2 \exp(- t^2 / K).$$ 
Via the continuity of the functions $h_1$ and $h_2$, we can then bound $L(Y)$ and $L_g(\theta,Y)$ with high probability as required in the statement. 
\end{proof}

\subsection{Gradient Descent Based Framework}
\label{sec:gradient_descent}
Since we already verified the $\lambda$-concavity, we only need to verify $\gamma$-gradient smoothness and $\mu$-smoothness, which we do in the following propositions.

\begin{proposition}[Gradient Smoothness Condition]
    Let \hyperref[ass:A1]{(A1)}, \hyperref[ass:A3]{(A3)}, \hyperref[ass:A4]{(A4)} and \hyperref[ass:A5]{(A5)} be fulfilled. Then $\mathcal{Q}$ satisfies the $\gamma$-gradient smoothness condition  with $\gamma$ from \hyperref[ass:A5]{(A5)}.
\end{proposition}

\begin{proof}
The proof follows exactly the same lines as those of Proposition \ref{prop4} with $\mathcal Q(\theta,\hat \theta)$ replaced by $\mathcal Q(\theta, \theta)$ 
and just (A5) applied in the final step.
\end{proof}

\begin{proposition}[$\mu$-Smoothness]\label{lemma:musmooth}
    Let \hyperref[ass:A1]{(A1)} and \hyperref[ass:A2]{(A2)} be fulfilled.  Then there exists $\mu>0$ such that
$\mathcal{Q}(\cdot,\theta^\ast)$ is strongly $\mu$-smooth on $B_\varepsilon(\theta^*)$.
\end{proposition}

\begin{proof}
    Under regularity assumptions of \hyperref[ass:A1]{(A1)} we can swap the following differentiation and integration, and applying subsequently \hyperref[ass:A2]{(A2)}, we get 
\begin{equation}
    \nabla_1^2 \mathcal{Q}(\theta^\ast,  \theta^\ast) = 
    \mathbb{E}_{y \sim p_{Y_{\theta^*}}} \mathbb{E}_{x \sim P_{X|Y_{\theta^*}=y}} \left[
    \nabla_{\theta}^2 \log p_{X,Y_{\theta^\ast}}(x,y)
    \right]  \geq -\mu^\ast \mathrm{Id}.
\end{equation}
By \hyperref[ass:A1]{(A1)} the map
$
\nabla_1^2 \mathcal{Q}(\cdot, \theta^\ast) 
$
is continuous on  $B_\varepsilon(\theta^*)$. Hence, for $\varepsilon>0$ small enough, $\mu$-smoothness follows. 
\end{proof}

\subsection{Step Size Choice Revisited }
\label{Sec:stepsize}

The result in \Cref{thm:main_2}, in particular the constraint to $\gamma< \lambda$ stems from the a prior choice of the step size which is optimal  in the classic gradient descent mechanism for $\lambda$-strong convex and $\mu$-smooth objectives. Alternatively, we can choose the step-size later and hence relax the assumptions slightly. This is the goal of this subsection. 

As a classic result for a (local) $\lambda$-strong concave and $\mu$-smooth function $q$ with unique maximizer $\theta^\ast$, the gradient ascent iterates 
$$
\theta^{(k+1)} = \theta^{(k)} + \tau \nabla q(\theta^{(k)}), \qquad \theta^{(0)}\in B_{\epsilon}(\theta^\ast)
$$
satisfy
$$
\|\theta^{(k+1)}-\theta^\ast \|^2 \leq \left( 1 - \frac{2\mu\lambda}{\mu+\lambda}\tau\right)\|\theta^{(k)}-\theta^\ast\|^2.
$$
Let $\mathcal{Q}$ be locally $\gamma$-gradient smooth. Then, we have for
$q(\theta):=\mathcal{Q}(\theta,\theta^\ast)$ that the EM-iterates $\theta^{(k)}$ fulfill 
\begin{align*}
    \|\theta^{(k+1)}-\theta^\ast \| &= 
    \|\theta^{(k)} + \tau \nabla_1 \mathcal{Q}(\theta, \theta^{(k)}) \big|_{\theta = \theta^{(k)}}-\theta^\ast \| \\
    &\leq  
    \|\theta^{(k)} + \tau \nabla q(\theta^{(k)}) -\theta^\ast\| +
    \tau \| \nabla_1
    \mathcal{Q}(\theta, \theta^{(k)}) \big|_{\theta = \theta^{(k)}} - \nabla q(\theta^{(k)})  \| \\
    &\leq \left(
    \sqrt{ 1 - \frac{2\mu\lambda}{\mu+\lambda}\tau } + \tau\gamma \right)\| \theta^{(k)} - \theta^\ast\|.
\end{align*}
Consequently, convergence can be assured if we find a step size $\tau>0$ such that 
\begin{align} \label{starr}
    0 < \sqrt{1- \tau c} + \tau \gamma < 1 \quad \textrm{for }
    c \coloneqq \frac{2\mu\lambda}{\lambda + \mu}.
\end{align}
We have the following lemma.

\begin{lemma}\label{lem:conv_gs}
    For any $\mu,\lambda>0$ and $\gamma$ such that $0\leq \gamma < c=\frac{2\mu\lambda}{\mu+\lambda} $ and $\theta^0 \in B_{\epsilon}(\theta^\ast)$ for $\epsilon$ small enough,  there exists a step size $\tau\in (0,c^{-1})$ such the gradient ascent iterates in the EM-Algorithm converges to $\theta^\ast$.
\end{lemma}

\begin{proof}
Since $\tau < c^{-1}$, the square root is in $(0,1)$. 
 For $\gamma = 0$ we are done. Assume that $\gamma >0$.
 Then \eqref{starr} implies immediately
$\tau< \gamma^{-1}$ so that
$\tau < \min\{ c^{-1},\gamma^{-1} \}$ and then
\begin{align*}
      (1-\tau c) < (1 - \tau \gamma)^2 
   \quad
    \Leftrightarrow 
    \quad
    \frac{2\gamma - c}{\gamma^2}  < \tau.
\end{align*}
Hence to get a non-empty set of step sizes,  we need
$$
\frac{2\gamma - c}{\gamma^2} < \min\left\{ c^{-1},\gamma^{-1} \right\}.
$$
While $\frac{2\gamma - c}{\gamma^2} < \frac{1}{c}$ is always fulfilled for $c \not = \gamma$,
the condition 
$\frac{2\gamma - c}{\gamma^2} < \frac{1}{\gamma}$ holds true if and only if 
$\gamma < c$.
\end{proof}

Since $\mu \geq \lambda$, it holds  
    $
    c  \geq \lambda,
    $
    with equality if and only if $\mu=\lambda$.
In particular in the case $\mu> \lambda$, we can relax the constraints on $\gamma$ from \cite{em_conv}, i.e. $0\leq \gamma < \lambda$, for convergence to the range $0\leq \gamma < c$ with $c$ from \eqref{starr}. The resulting improvement is illustrated in \Cref{fig:c_vs_lambda}.

\begin{figure}
    \centering
    \includegraphics[width=0.65\linewidth]{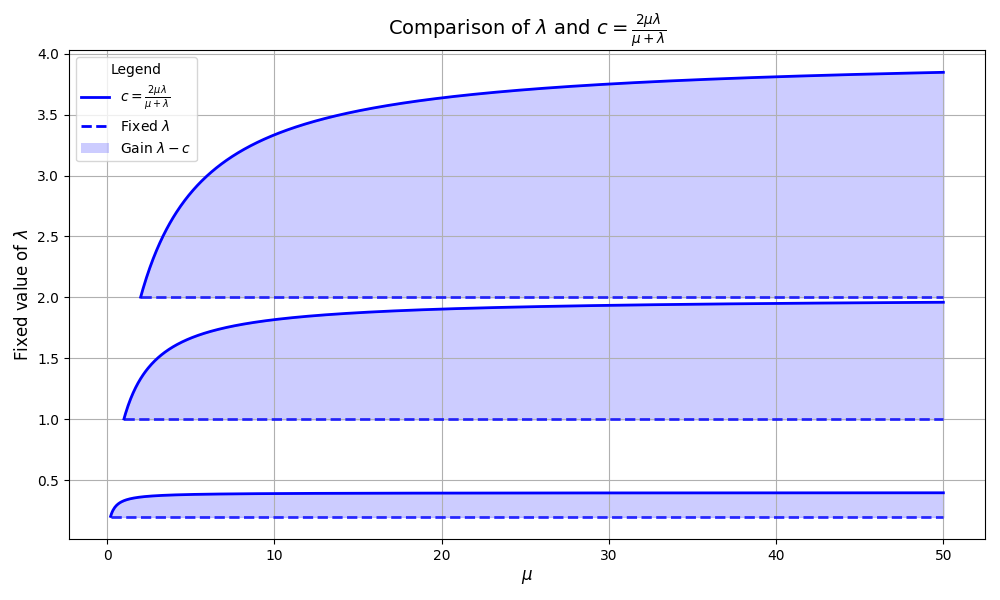}
    \caption{Comparison of $c=\frac{2\mu\lambda}{\mu+\lambda}$ and $\lambda$ for fixed values of $\lambda =0.2,1,2$ for a range of $\mu > \lambda$.}
    \label{fig:c_vs_lambda}
\end{figure}

\section{Verification of Assumptions for Mixed Noise Model}
\label{sec:verification}
Here, we verify the assumptions of our mixed noise error model, in order to apply our proposed theory. 
Throughout the analysis, we impose the following additional assumptions to ensure well-posedness of the model and validity of the theoretical results:
\begin{itemize}
    \item \textbf{Regularity and growth of the forward model:} Each component \( f_i: \mathbb{R}^{d} \to \mathbb{R} \) of $F=(f_1,\ldots,f_n)^\top$ for $i=1,\ldots,n$ is continuously differentiable, bounded,  and satisfies 
    the bounded gradient conditions:
    \[
    \|\nabla f_i(x)\| \leq C_{f'} \quad \text{for all } x \in \mathbb{R}^{d}.
    \]

   \item \textbf{Moment bounds:} Existence of \textit{first} and \textit{second} moment in $Y$-space under the joint distribution,
$$\int_{\R^{d \times n}}  |y|^k \, p_{X|Y_\theta=y} \,  p_{Y_{\theta^*}} \d x \d y< \infty,
$$
for k = 0,1 and 2.

    \item \textbf{Prior regularity:} The logarithm of the prior density \( p_X \) is Lipschitz continuous and has finite moments up to order 4:
    \[
    \mathbb{E}_{x \sim p_X}[\|x\|^4] < \infty.
    \]

    \item \textbf{Compact parameter domain:} The parameter space \( \Theta \subset \mathbb{R}^2 \) is compact with strictly positive lower bounds
    \[
    a^2 \in [a_{\min}, a_{\max}], \quad b^2 \in [b_{\min}, b_{\max}], \quad \text{with } a_{\min}, b_{\min} > 0.
    \]

    \item \textbf{Non-degeneracy of the forward map:} For at least one \( i \), the map \( f_i \) is non-constant on the support of \( p_X \), ensuring identifiability of the noise components and strict positivity of the Fisher information matrix.
\end{itemize}

We consider the space 
$$\Theta \coloneqq \{\theta = (a^2, b^2): 
a^2 \in [a_{\text{min}}, a_{\max}], b^2 \in [b_{\text{min}}, b_{\max}] \}
$$ 
where $a_{min}, b_{min} >0$. 
Defining for $F = (f_1,\ldots,f_n)^\tT$
and $y = (y_1,\ldots,y_n)^\tT$ the functions
$$
\sigma_i (x) \coloneqq a^2 + b^2 f_i(x)^2, \quad 
r_i(x) \coloneqq r_i(x,y)= y_i-f_i(x)
$$
the log likelihood is given by 
\begin{align*}
\log p_{Y_\theta|X=x}(y)
&= \log \mathcal N \left( y|(F(x), a^2 I_n + b^2 \text{diag} (F(x)^2 \right)\\
&=
-\frac{1}{2}\left(\sum_{i=1}^n \log (2 \pi \sigma_i(x)) + \frac{r_i(x,y)^2}{\sigma_i(x) }\right).
\end{align*}

\paragraph{Assumption (A1)}

To verify assumption \hyperref[ass:A1]{(A1)}, we calculate the derivatives of the log-likelihood
\[
\log p_{Y_\theta|X=x}(y) = -\frac{1}{2} \sum_{i=1}^n \left( \log (2 \pi \sigma_i(x)) + \frac{r_i(x,y)^2}{\sigma_i(x)} \right),
\]
where \( r_i(x,y) := y_i - f_i(x) \) and \( \sigma_i(x) := a^2 + b^2 f_i(x)^2 \). Then:

\begin{align*}
\partial_{a^2} \log p_{Y_\theta|X=x}(y) 
&= -\frac{1}{2} \sum_{i=1}^n \left( \frac{1}{\sigma_i(x)} - \frac{r_i(x,y)^2}{\sigma_i(x)^2} \right), \\
\partial_{b^2} \log p_{Y_\theta|X=x}(y) 
&= -\frac{1}{2} \sum_{i=1}^n f_i(x)^2 \left( \frac{1}{\sigma_i(x)} - \frac{r_i(x,y)^2}{\sigma_i(x)^2} \right).
\end{align*}

For the second-order partial derivatives (Hessian), we obtain:
\begin{align*}
\partial_{a^2}^2 \log p_{Y_\theta|X=x}(y) 
&= \frac{1}{2} \sum_{i=1}^n \left( \frac{1}{\sigma_i(x)^2} - \frac{2 r_i(x,y)^2}{\sigma_i(x)^3} \right), \\
\partial_{a^2} \partial_{b^2} \log p_{Y_\theta|X=x}(y) 
&= \frac{1}{2} \sum_{i=1}^n f_i(x)^2 \left( \frac{1}{\sigma_i(x)^2} - \frac{2 r_i(x,y)^2}{\sigma_i(x)^3} \right), \\
\partial_{b^2}^2 \log p_{Y_\theta|X=x}(y) 
&= \frac{1}{2} \sum_{i=1}^n f_i(x)^4 \left( \frac{1}{\sigma_i(x)^2} - \frac{2 r_i(x,y)^2}{\sigma_i(x)^3} \right).
\end{align*}

We assume that the required moments exist under the distribution \( p_{X|Y_\theta = y} \, p_{Y_{\theta^*}} \).

Since the log-likelihood is smooth in \( \theta \in \Theta \), and the derivatives involve rational expressions in \( f_i(x) \), \( r_i(x,y) \), and \( \sigma_i(x) \), all of which are smooth functions of \( x \) and \( y \), the gradient and Hessian are continuous. 
Furthermore, under moment bounds on \( y \), the gradient and Hessian norms are integrable, since all involved components are uniformly bounded in $x$ as rational polynomials in $f_i(x)$ with nominator of degree larger or equal to the nominators.

For example, let $z=f_i(x)\in\mathbb{R}$. Then 
\begin{align*}
 \left|\frac{1}{\sigma_i(x)^2} - \frac{2 r_i(x,y)^2}{\sigma_i(x)^3}\right| &=
 \left| 
    \frac{1}{(a^2+b^2 z^2)^2} -  \frac{2z^2}{(a^2+b^2z^2)^3} + \frac{4zy_i - 2y_i^2}{(a^2+b^2z^2)^3} 
 \right|
 \\
 &\leq 
 \underbrace{\left|
\frac{1}{(a^2+b^2 z^2)^2} -  \frac{2z^2}{(a^2+b^2z^2)^3}
 \right|}_{\leq c_0 < \infty \forall z\in\mathbb{R}}
 + \underbrace{\left| \frac{4z}{(a^2+b^2z^2)^3}\right|}_{\leq c_1 < \infty \forall z\in\mathbb{R}} |y_i| + \underbrace{\left|\frac{2}{(a^2+b^2z^2)^3} \right|}_{\leq c_2<\infty\forall z\in\mathbb{R}}|y_i|^2 \\
 &\leq c_0 + c_1 |y_i| + c_2 |y_i|^2.
\end{align*}
Consequently, for each summand exemplarily for the $\partial_{a^2}$ term, we find
$$
\int_{\R^{d \times n}} \frac{1}{\sigma_i(x)^2} - \frac{2 r_i(x,y)^2}{\sigma_i(x)^3} \, p_{X|Y_\theta=y} \,  p_{Y_{\theta^*}} \d x \d y \leq \int_{\R^{d \times n}} c_0 + c_1 |y_i| + c_2 |y_i|^2 \, p_{X|Y_\theta=y} \,  p_{Y_{\theta^*}}\d x \d y < \infty.
$$
The other bounds for the majorants are obtained analogously, using the gradient condition on the forward operator $f$.

Therefore, the conditions of Assumption \assref{A1}—namely, differentiability of \( \theta \mapsto \log p_{X,Y_\theta}(x,y) \) and integrable bounds on its gradient and Hessian—are satisfied.
\paragraph{Assumption (A2)}
\paragraph{Existence of $\lambda_{\ast}>0$:} We derive conditions to ensure the existence of a uniform lower bound $\lambda_{\ast} > 0$ for the Fisher information matrix.

Recall that the likelihood under the mixed noise model is given by
\[
p_{Y|X=x}(y) \propto \prod_{i=1}^n \exp\left(-\frac{1}{2} \left[\log \sigma_i(x) + \frac{(y_i - f_i(x))^2}{\sigma_i(x)} \right] \right)
= \prod_{i=1}^n \frac{1}{\sqrt{\sigma_i(x)}} \exp\left( -\frac{1}{2} \frac{(y_i - f_i(x))^2}{\sigma_i(x)} \right),
\]
where we define $\sigma_i(x) := a^2 + b^2 f_i(x)^2$ and $r_i(x, y) := y_i - f_i(x)$. Consequently,
\[
p_{Y|X=x}(y) = \mathcal{N}(F(x), \bm{\Sigma}(x)), \quad F(x) = (f_i(x))_{i=1}^n,\quad \bm{\Sigma}(x) = \operatorname{diag}(\sigma_1(x), \dots, \sigma_n(x)).
\]
From this, we obtain
\[
\mathbb{E}_{Y|X=x}[r_i(x,y)] = 0, \quad \mathbb{E}_{Y|X=x}[r_i(x,y)^2] = \sigma_i(x).
\]

Now consider the joint density $p_{Y,X}(y, x \mid \theta) = p_{Y|X=x}(y \mid \theta) p_X(x)$. Since the prior $p_X$ is independent of $\theta$, the Fisher information matrix becomes
\[
\mathcal{I}(\theta) := -\mathbb{E}_{X,Y}[\nabla^2_\theta \log p_{X,Y_{\theta}}(x,y)] = -\mathbb{E}_X \left[ \mathbb{E}_{Y|X=x}[\nabla^2_\theta \log p_{Y_{\theta}|X=x}(y)] \right].
\]
Using the derivatives computed in the verification of Assumption (A1), we obtain:
\[
\mathcal{I}(\theta) = \frac{1}{2} \sum_{i=1}^n \int_{\mathbb{R}^{d}} 
\frac{1}{\sigma_i(x)^2}
\begin{pmatrix}
1 & f_i(x)^2 \\
f_i(x)^2 & f_i(x)^4
\end{pmatrix}
p_X(x)\, \mathrm{d}x.
\]

Define for each $i$:
\[
A_i := \int \frac{1}{\sigma_i(x)^2} p_X(x)\, \mathrm{d}x, \quad 
B_i := \int \frac{f_i(x)^2}{\sigma_i(x)^2} p_X(x)\, \mathrm{d}x, \quad 
C_i := \int \frac{f_i(x)^4}{\sigma_i(x)^2} p_X(x)\, \mathrm{d}x,
\]
and set
\[
A = \sum_{i=1}^n A_i, \quad B = \sum_{i=1}^n B_i, \quad C = \sum_{i=1}^n C_i.
\]
Then the Fisher information matrix becomes:
\[
\mathcal{I}(\theta) = \frac{1}{2} 
\begin{pmatrix}
A & B \\
B & C
\end{pmatrix}.
\]

A symmetric $2 \times 2$ matrix is positive definite if and only if its leading principal minors are positive:
\[
A > 0 \quad \text{and} \quad AC - B^2 > 0.
\]

\begin{enumerate}
\item \textbf{Positivity of $A$:} Since $a^2 + b^2 f_i(x)^2 > 0$ for all $x$ (because $a^2 > 0$ and $b^2 \geq 0$), we have $1/\sigma_i(x)^2 > 0$. If $p_X$ has full support and $(a^2, b^2) \in \Theta$ is bounded away from zero (i.e., $a_{\min} > 0$), then each $A_i > 0$, hence $A > 0$.

\item \textbf{Strict positivity of $AC - B^2$:} 
Apply the Cauchy–Schwarz inequality to each $i$:
\[
B_i^2 = \left( \int \frac{f_i(x)^2}{\sigma_i(x)^2} p_X(x)\, \mathrm{d}x \right)^2 
\leq \left( \int \frac{1}{\sigma_i(x)^2} p_X(x)\, \mathrm{d}x \right)
\left( \int \frac{f_i(x)^4}{\sigma_i(x)^2} p_X(x)\, \mathrm{d}x \right) 
= A_i C_i,
\]
with equality if and only if $f_i(x) \equiv \mathrm{const}$. Summing over $i$ and applying Cauchy–Schwarz again:
\[
B^2 = \left( \sum_i B_i \right)^2 \leq \left( \sum_i \sqrt{A_i C_i} \right)^2 \leq AC,
\]
with strict inequality if $f_i \not\equiv \mathrm{const}$ for at least one $i$. This is generically satisfied in inverse problems, as constant $f$ would prevent identifiability of noise components.
\end{enumerate}

\paragraph{Assumption (A3)}

To verify Assumption \hyperref[ass:A3]{(A3)}, we follow the stability result in \cite[Theorem 14]{sprungk2020local} in order to obtain Lipschitz continuity of the posterior \( \theta \mapsto P_{X|Y_\theta=y} \) in the Wasserstein-1 distance.

The key requirement is to bound the difference between log-likelihoods, i.e.,
\begin{align*}
&\left\| \log p_{Y_{\theta_1}|X=x}(y) - \log p_{Y_{\theta_2}|X=x}(y) \right\|_{L^p(P_X)} := \\
&\left( \int_{\mathbb{R}^d} \left| \log p_{Y_{\theta_1}|X=x}(y) - \log p_{Y_{\theta_2}|X=x}(y) \right|^p p_X(x) \, \mathrm{d}x \right)^{1/p}
\end{align*}
for \( p = 1, 2 \). 
To estimate this, we first define a bound on the derivatives as
\[
\mathrm{maxgrad}\ l(x,y) := \max\left\{ \left| \partial_{a^2} \log p_{Y_{\theta}|X=x}(y) \right|, \left| \partial_{b^2} \log p_{Y_{\theta}|X=x}(y) \right| \right\}.
\]
Since $a^2, b^2$ are bounded, there exists  $C_{a^2}(y_i)=C_{a^2}(y_i, a_{\mathrm{min}}, a_{\mathrm{max}}, b_{\mathrm{min}},b_{\mathrm{max}})>0$ and $C_{b^2}(y_i)=C_{b^2}(y_i, a_{\mathrm{min}}, a_{\mathrm{max}}, b_{\mathrm{min}},b_{\mathrm{max}})>0$ such that  for all $z \in \R$,
$$
0\leq \left|\frac{1}{(a^2+b^2z^2)} - \frac{(y_i-z)^2}{(a^2+b^2z^2)^2} \right| < C_{a^2}(y_i) = C_{a^2}(y_i)
$$
and
$$
0 \leq \left|
z^2\left(\frac{1}{(a^2+b^2z^2)} - \frac{(y_i-z)^2}{(a^2+b^2z^2)^2}\right)
\right|
< C_{b^2}(y_i) = C_{b^2}(y_i),
$$
where we used the fact that each denominator is a polynomials in $z$ of order higher or equal then the polynomial in the nominator.  Consequently, independent of $x$ we can bound
\[
\mathrm{maxgrad}\ l(x,y) \leq \sup_{\theta}
\frac{1}{2}\sum\limits_{i=1}^n \max\{ C_{a^2}(y_i), C_{b^2}(y_i)\} =: C_{\mathrm{grad}}(y).
\]
Hence, by the mean value theorem in parameter space
\[
\left| \log p_{Y_{\theta_1}|X=x}(y) - \log p_{Y_{\theta_2}|X=x}(y) \right| \leq C_{\mathrm{grad}}(y) \cdot \| \theta_1 - \theta_2 \|.
\]
As a result,
\[
\left\| \log p_{Y_{\theta_1}|X=x}(y) - \log p_{Y_{\theta_2}|X=x}(y) \right\|_{L^p(P_X)} \leq L(y) \| \theta_1 - \theta_2 \|,
\]
for \( p = 1, 2 \) constant \( L(y) \).
The derivation further shows that \( L(y) \) depends on $ y $ only as a polynomial of degree $2$ of $ \| y \| $.
Following the proof of \cite[Theorem 14]{sprungk2020local}, we now obtain that
\begin{align}
  W_1(P_{X|Y_{\theta_1}=y}, P_{X|Y_{\theta_2}=y}) \leq \frac{L(y) \mathbb{E}_{x \sim p_X}[\| x \|^2] }{\min(Z_1, Z_2)}
                                                       \| \theta_1 - \theta_2 \|,
\end{align}
where $ Z_1 $ and $ Z_2 $ are the normalization constants \( Z_1, Z_2 \) of the posterior densities.
Note that the assumption \( \text{essinf}_{p_X} \log p_{Y_{\theta_1} | X = x} = 0 \) in the proof can be avoided by considering \( \exp(-[\text{essinf}_{p_X} \min(\log p_{Y_{\theta_1} | X = x}, \\ \log p_{Y_{\theta_2} | X = x})]_{-}) \le 1 \) in constant.

Finally, the normalization constants can be estimated by
\begin{align}
  \min(Z_1, Z_2) =   \min_{k = 1, 2} 
                     \int \exp(- \log p_{Y_{\theta_k} | X = x}) p_X(x) d x
                 \ge \exp(-C \| y \|^2)
\end{align}
for a constant only depending on \(a_{\min}, a_{\max}, b_{\min}, b_{\max} \).
Absorbing \( \exp(-C \| y \|^{2}) \) into \( L(y) \) and using the moment assumption on the prior yields that
\begin{align}
  W_{1}(P_{X | Y_{\theta_{1}} = y}, P_{X | Y_{\theta_{2}} = y}) \le L(y) \| \theta_{1} - \theta_{2} \| 
\end{align}
for \( L(y) \) that is of size \( C (\| y \|^{4} + e^{C \| y \|^{2}}) \) with a suitable constant \( C > 0 \).
This verifies Assumption \assref{A3}.

\paragraph{Assumption (A4)}
We establish a Lipschitz bound on the gradient of the complete-data log-likelihood \( \nabla_{\theta} \log p_{X,Y_\theta}(y,x) \) with respect to the data variable \( x \in \mathbb{R}^{d} \).
\\[1ex]
\textbf{Difference in \( \partial_{a^2} \) Terms.}
Let
\[
g_i(x) := \frac{1}{\sigma_i(x)} - \frac{r_i(x)^2}{\sigma_i(x)^2}, \quad \text{with } r_i(x) := y_i - f_i(x).
\]
Then,
\begin{align*}
&\left| \partial_{a^2} \log p_{X,Y_\theta}(x_1,y) - \partial_{a^2} \log p_{X,Y_\theta}(x_2,y) \right| \\
&\leq \left| \partial_{a^2} \log p_{Y_\theta|X=x_1}(y) - \partial_{a^2} \log p_{Y_\theta|X=x_2}(y) \right| + \left| \log p_X(x_1) - \log p_X(x_2) \right| \\
&\leq \frac{1}{2} \sum_{i=1}^n |g_i(x_1) - g_i(x_2)| + \left| \log p_X(x_1) - \log p_X(x_2) \right|.
\end{align*}
To bound \( |g_i(x_1) - g_i(x_2)| \), apply the mean value theorem. The gradient reads:
\begin{align*}
\nabla_x g_i(x) &= -\frac{2b^2 f_i(x)}{\sigma_i(x)^2} \nabla_x f_i(x) 
+ \frac{4b^2 f_i(x) r_i(x)^2}{\sigma_i(x)^3} \nabla_x f_i(x) 
+ \frac{2r_i(x)}{\sigma_i(x)^2} \nabla_x f_i(x) \\
&= \left[ -\frac{2b^2 f_i(x)}{\sigma_i(x)^2} + \frac{4b^2 f_i(x) r_i(x)^2}{\sigma_i(x)^3} + \frac{2r_i(x)}{\sigma_i(x)^2} \right] \nabla_x f_i(x).
\end{align*}
Let $f_i(x) = z$ and define
$$
h(z) = 
-\frac{2b^2 z}{(a^2 +b^2 z^2)^2} + \frac{4b^2 z (y_i-z)^2}{(a^2+b^2 z^2)^3} + \frac{2(y_i-z)}{(a^2+b^2z^2)^2}
$$
It holds $a^2+b^2z^2 \geq a_{\mathrm{min}}>0$ for all $z\in\mathbb{R}$ and $|h(z)|\to 0$ for $|z|\to\infty$. Consequently $h(z)$ attains a finite maximum, i.e. there exists $C_h = C_h(a^2,b^2,y_i)>0$ with $|h(z)|\leq C_h< \infty$ for all $z\in\mathbb{R}$. Together with the uniform bound assumption on $\nabla_x f_i(x)$, we obtain 
$$
||\nabla_x g_i(x)|| \leq C_h(a^2,b^2,y_i)C_{f'} =:C_{i,1}(y_i,\theta).
$$
\textbf{Difference in \( \partial_{b^2} \) Terms.}
Let
\[
h_i(x) := f_i(x)^2 \cdot g_i(x).
\]
Then we get
\[
\left| \partial_{b^2} \log p_{X,Y_\theta}(x_1,y) - \partial_{b^2} \log p_{X,Y_\theta}(x_2,y) \right| 
\leq \frac{1}{2} \sum_{i=1}^n |h_i(x_1) - h_i(x_2)| + \left| \log p_X(x_1) - \log p_X(x_2) \right|.
\]
Let again $f_i(x) = z$, then $g_i(x) = \frac{1}{a^2+b^2z^2} - \frac{(y_i-z)^2}{(a^2+b^2z^2)^2}$. We bound  $|h_i(x_1)-h_i(x_2)|$ with mean value theorem and therefore define 
$$
H(z) = z^2 \left( \frac{1}{a^2 + b^2 z^2} - \frac{(y_i - z)^2}{(a^2 + b^2 z^2)^2} \right).
$$
Then 
\[
\begin{aligned}
H'(z) 
&= 2z \left( \frac{1}{a^2 + b^2 z^2} - \frac{(y_i - z)^2}{(a^2 + b^2 z^2)^2} \right) 
 + z^2 \left( \frac{-2b^2 z}{(a^2 + b^2 z^2)^2}
+ \frac{2(y_i - z)}{(a^2 + b^2 z^2)^2}
+ \frac{4b^2 z (y_i - z)^2}{(a^2 + b^2 z^2)^3} \right)
\end{aligned}
\]
and we conclude $|H'(z)| \in \mathcal{O}(|z|^{-1})$ for $|z|\to\infty$. Hence we have $|H'(z)|\le C_{i,2}(y_i,\theta)$ for a constant $C_{i,2}>0$ independent of $z=f_i(x)$ and in particular of $x$. 
This yields 
$$
|h_i(x_1) - h_i(x_2)| \leq C_{i,2}(y_i,\theta) \|x_1-x_2\|.
$$
\\[1ex]
\textbf{Combining both terms.}
Now, the gradient difference satisfies
\begin{align*}
&\left\| \nabla_{\theta} \log p_{X,Y_{\theta}}(x_1,y) - \nabla_{\theta} \log p_{X,Y_{\theta}}(x_2,y) \right\|^2 
= |\partial_{a^2} \cdot |^2 + |\partial_{b^2} \cdot |^2 \\
&\leq \left( \frac{1}{2} \sum_{i=1}^n C_{i,1}(y_i, \theta) \|x_1 - x_2\| + |\log p_X(x_1) - \log p_X(x_2)| \right)^2 \\
&\quad + \left( \frac{1}{2} \sum_{i=1}^n C_{i,2}(y_i, \theta) \|x_1 - x_2\| + |\log p_X(x_1) - \log p_X(x_2)| \right)^2.
\end{align*}
Assume that \( \log p_X \) is Lipschitz with constant \( L_0 \). Then it holds
\[
|\log p_X(x_1) - \log p_X(x_2)| \leq L_0 \|x_1 - x_2\|.
\]
We define
\[
C_{total}(y,\theta) := \left( \sum_{i=1}^n \frac{1}{2} C_{i,1}(y_i, \theta) + L_0 \right)^2 + \left( \sum_{i=1}^n \frac{1}{2} C_{i,2}(y_i, \theta) + L_0 \right)^2.
\]
Then we obtain the Lipschitz bound\[
\left\| \nabla_{\theta} \log p_{X,Y_{\theta}}(x_1,y) - \nabla_{\theta} \log p_{X,Y_{\theta}}(x_2,y) \right\| 
\leq \sqrt{C_{total}(y, \theta)} \|x_1 - x_2\|.
\]
This verifies Assumption \hyperref[ass:A4]{(A4)}.
    
\paragraph{Assumption (A5-SG)}

To verify assumption \hyperref[ass:A5SG]{(A5-SG)}, we first need to show that $Y_{\theta}$ is indeed sub-Gaussian. 
First note that,  
$$\mathbb{E}[\Vert Y_{\theta} \Vert^p]^{1/p} \leq \mathbb{E}[\Vert F(X) \Vert^p]^{1/p} + \mathbb{E}[ \Vert \eta_{\theta} \Vert^p]^{1/p}. $$
Now, due to the boundedness of $F$ and the bound constraints on the parameters $a,b$, there exists a $C\neq C(p)>0$ such that 
$$\mathbb{E}[\Vert F(X) \Vert^p]^{1/p} + \mathbb{E}[ \Vert \eta_{\theta} \Vert^p]^{1/p}\leq C \sqrt{p}.$$
Then, the characterization of Sub-gaussian properties, in particular the moment property in \cite[Proposition 2.6.1]{Vershynin_2018}, yield the sub-Gaussianity of $Y_{\theta}$. 
Finally, the bounds in \hyperref[ass:A3]{(A3)} and \hyperref[ass:A4]{(A4)} guarantee the existence of the required continuous functions $h_1$ and $h_2$.

\begin{remark}
With similar arguments and assumptions in \ref{ass:A3} to \cite{stuart_2010, sprungk2020local} one can also show local Lipschitz continuity for the total variation distance. Then \hyperref[ass:A4]{(A4)} is needs to be exchanged by boundedness of the score $\nabla_{\theta} \log p_{X,Y_{\theta}}$. 
\end{remark}
\begin{remark}\label{rem:gamma_c}
Unfortunately, the bounds in the Wasserstein case give an exponential factor, which destroys the integrability required in \hyperref[ass:A5]{(A5)}. Similar to \cite[Corollary 19]{sprungk2020local}, we only get convergence on sets where $Y$ has bounded radius. This high probability statement does not yield the EM convergence. However, the probabilistic FOS argument gives plausibility to convergence in practice, as will be shown in the next section. 
\end{remark}

\section{Numerical Examples} \label{sec:numerics}

In this section, we aim to show that the resulting algorithm works remarkably well in high dimensions, i.e., on imaging examples. For this, we modify the algorithm from \cite{Hagemann_2024} by replacing their conditional normalizing flows with flow matching. The flow matching framework \cite{lipman2023flow} has been shown to be very efficient. 

First, we modify the algorithm of \cite{Hagemann_2024} in two ways, namely that we replace the learning in the E-step by flow matching and further, the learning of the noise parameters $(a,b)$ that was done by another EM in \cite{Hagemann_2024} by a gradient descent based scheme. The algorithm is summarized in Alg. \ref{alg_cnf_mixed}.

\begin{algorithm}[!ht]
\caption{EM Algorithm (Mixed noise estimation via flow matching, one observation)}\label{alg_cnf_mixed}
\begin{algorithmic}[1]  
\State \textbf{Input:} $\tilde y \in \mathbb{R}^n$ and initial estimate $\theta_0$ by some initial guess $(a, b)$
\For{$r = 0, 1, \ldots, R$} 
  \State \textbf{E-Step:} Simulate $x \sim P_X$, and $y \sim P_{Y_{\theta_r} \mid X = x}$
  \State Train the velocity field on the loss
  \[
  L(w) = \mathbb{E}_{z,x,y,t}[\Vert v_t(y,x_t) - (x-z) \Vert^2]
  \]
  \State \textbf{M-Step:} Simulate from the current velocity ODE: $x \sim P_{X \mid Y = \tilde{y}}$
  \State Calculate the loss
  \[
  \tilde{L}(\theta) = \mathbb{E}_{x \sim P_{X \mid Y_{\theta_r}}}\left[ -\log p_{Y_{\theta} \mid X = x}(\tilde{y}) \right]
  \]
  \State Update $\theta$ using gradient descent:
  \[
  \theta^{(r+1)} = \theta^{(r)} - \eta \nabla_{\theta} \tilde{L}(\theta)
  \]
\EndFor
\end{algorithmic}
\end{algorithm}

\subsection{MNIST Denoising}
\label{sec:MNIST}
We apply our flow matching noise estimation algorithm to MNIST \cite{mnist} denoising. For this we choose $a_{true} = 0.1$ and $b_{true} = 0.3$ and rescale the MNIST images to $[-1,1]$. We generate 16 noisy measurements, as can be seen on the left of Fig. \ref{fig:samples_fm}. Then we iterate between the flow matching algorithm for the current estimates of $(a,b)$ and then estimate $(a,b)$ based on the measurements. For the final flow matching model, one can see the posterior (aka clean) samples of the flow matching model on the right of Fig. \ref{fig:samples_fm}. We show the convergence of a and b over the epochs in Fig. \ref{fig:mnist_convab}. Here, $a$ decreases towards the true value and b increases towards the true value. The initial estimation was both for a and b at $\frac{1}{2}$.
\begin{figure}
    \centering
    \includegraphics[width=0.3\linewidth]{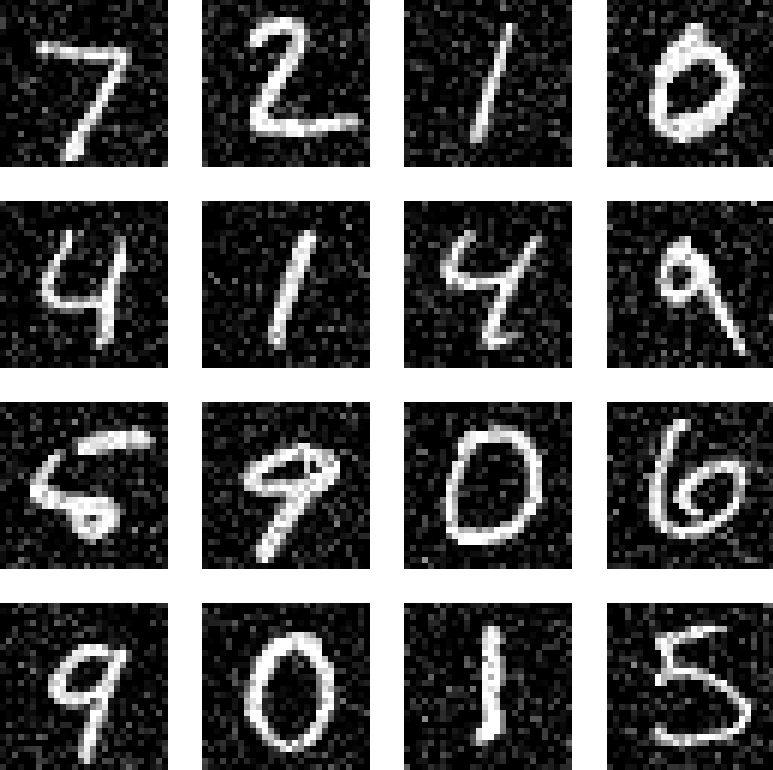}
        \hspace{0.05\linewidth}\includegraphics[width=0.3\linewidth]{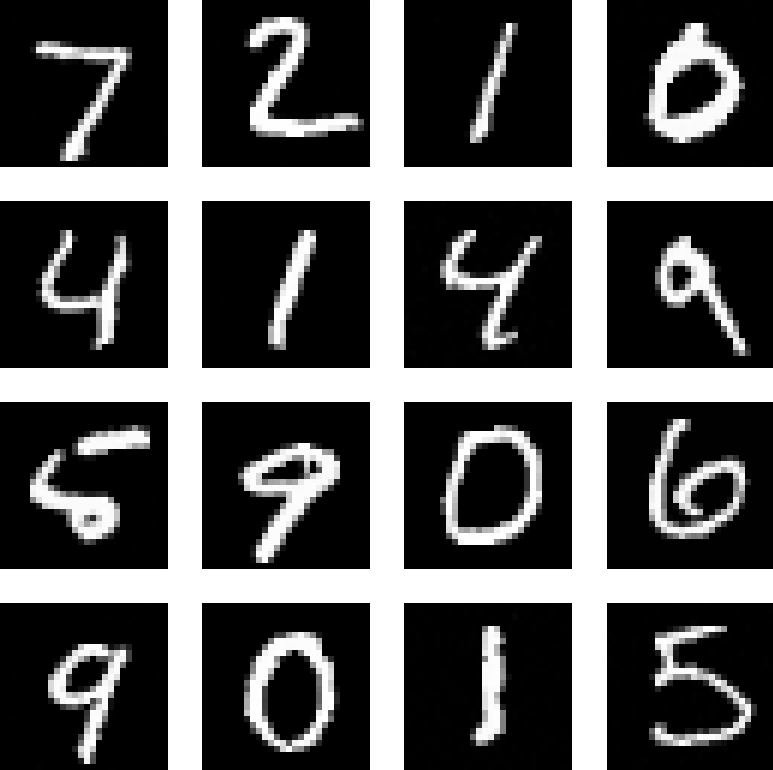}
    \caption{Conditions and resulting flow matching samples.}
    \label{fig:samples_fm}
\end{figure}

\begin{figure}
    \centering
    \includegraphics[width=0.44\linewidth]{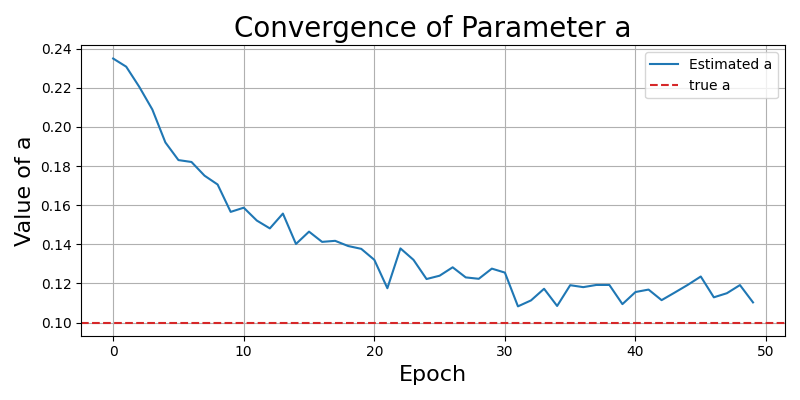}
    \includegraphics[width=0.44\linewidth]{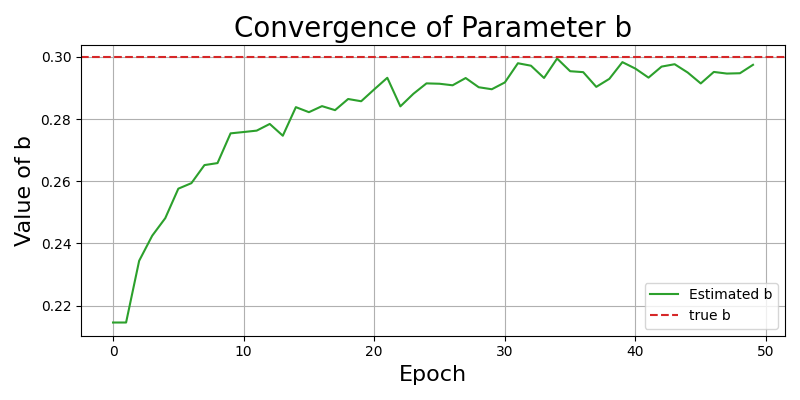}
    \caption{Convergence of a and b for the MNIST denoising imaging example.}
    \label{fig:mnist_convab}
\end{figure}

\subsection{Reaction-Diffusion PDE}
\label{sec:PDE}

In this example, we consider the forward model of a reaction diffusion equation. The forward operator is the solution of the PDE between time $50$ and $100$.  The initial distribution of the data at $t = 0$ is Gaussian. The reaction diffusion equation has two components $(u,v)$ and reads as 
\begin{align*}\partial_t u &= \alpha_u \partial_{x,x} u + \alpha_u \partial_{y,y} u + R_u(u,v),\\
\partial_t v &= \alpha_v \partial_{x,x} v + \alpha_v \partial_{y,y} v + R_v(u,v),
\end{align*}
where $\alpha_{u}$ and $\alpha_v$ are the diffusion coefficients and $R_u(u,v) = u - u^3-k-v$, $R_v(u,v) = u-v$ and $\beta_v$ are reaction functions \cite{takamoto2022pdebench}.

We approximate the forward operator based on the PDE bench dataset as provided in \cite{takamoto2022pdebench} using a Fourier Neural Operator \cite{li2021fourier,kossaifi2024neural}. Further, our velocity field in the flow matching framework is also parameterized by an FNO, see also \cite{hagemann2023multilevel}. By this, we mean that we concatenate the time input with the images in a channel and use a standard FNO. 

The choice of the neural operator does not correspond to an infinite dimensional framing of this inverse problem, but rather that the FNO has shown to be effective in modelling PDE solutions \cite{li2021fourier, takamoto2022pdebench}.

We train our flow matching algorithm in conjunction with the $(a,b)$ estimation for 30000 optimizer steps, where we update $(a,b)$ every 200 steps. We only consider a single image from a validation dataset as the condition, which is the second from left in Fig. \ref{fig:pde_plots}. One can see that the posterior mean and ground truth (first and third image) match quite well. The posterior standard deviation seems a bit noisy, but is larger where the forward operator is large. 

Similarly, the convergence of a and b can be observed in Fig. \ref{fig:pde_convab}. Here, the convergence seems more stochastic than in the MNIST example, which we attribute to the larger images and therefore worse posterior approximation. 

\begin{figure}[htbp]
\centering

\begin{subfigure}[b]{0.23\linewidth}
    \includegraphics[width=\linewidth]{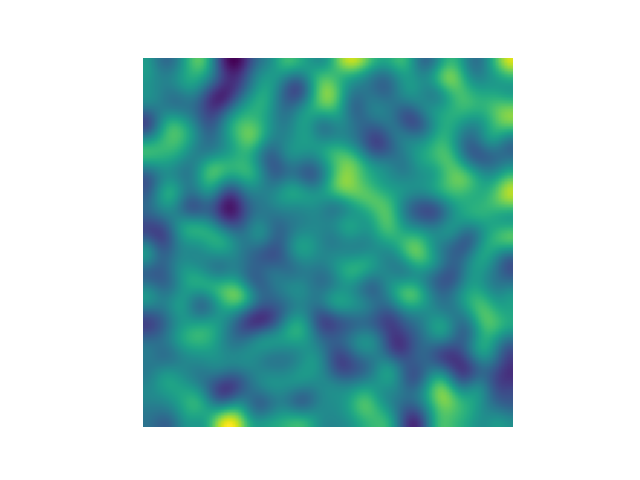}
    \caption{Ground truth $u$}
\end{subfigure}
\begin{subfigure}[b]{0.23\linewidth}
    \includegraphics[width=\linewidth]{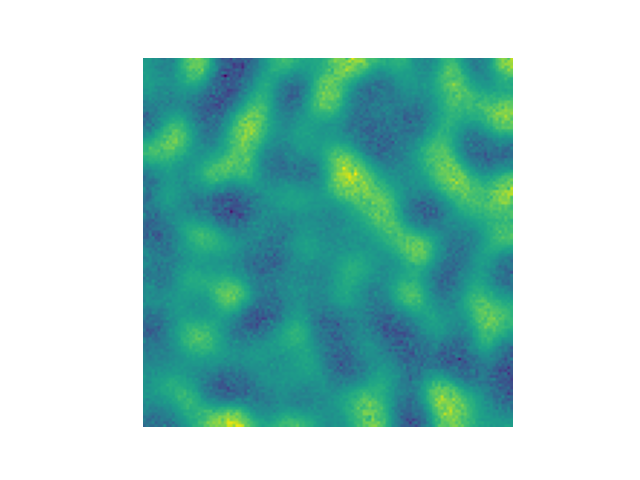}
    \caption{Measurement $u$}
\end{subfigure}
\begin{subfigure}[b]{0.23\linewidth}
    \includegraphics[width=\linewidth]{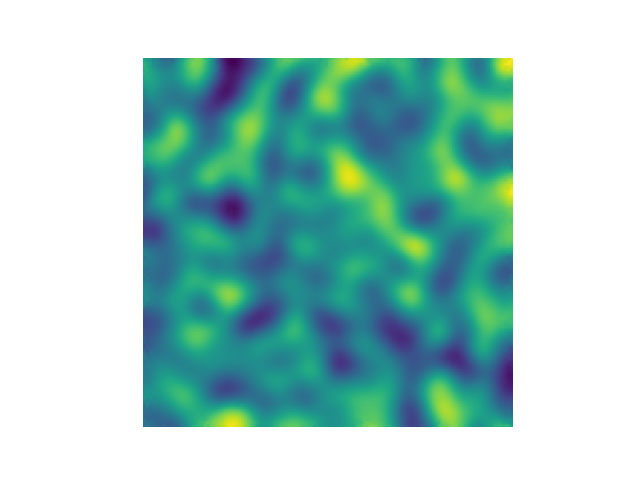}
    \caption{Posterior mean $u$}
\end{subfigure}
\begin{subfigure}[b]{0.23\linewidth}
    \includegraphics[width=\linewidth]{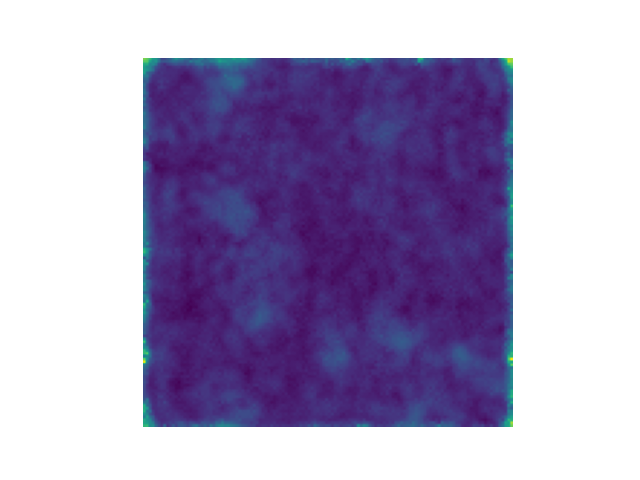}
    \caption{Posterior std. $u$}
\end{subfigure}

\vspace{0.5em}

\begin{subfigure}[b]{0.23\linewidth}
    \includegraphics[width=\linewidth]{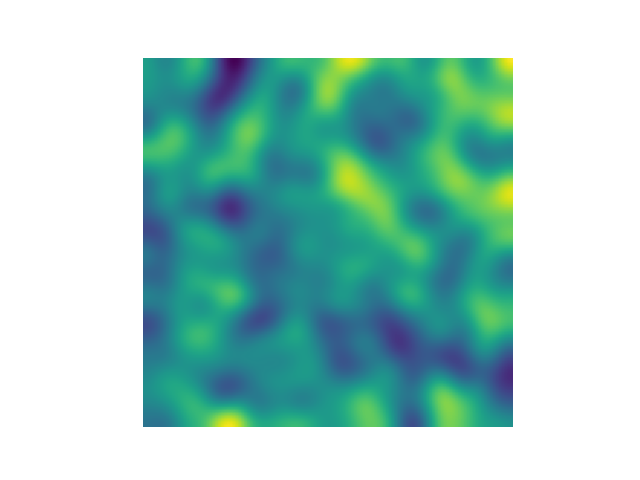}
    \caption{Ground truth $v$}
\end{subfigure}
\begin{subfigure}[b]{0.23\linewidth}
    \includegraphics[width=\linewidth]{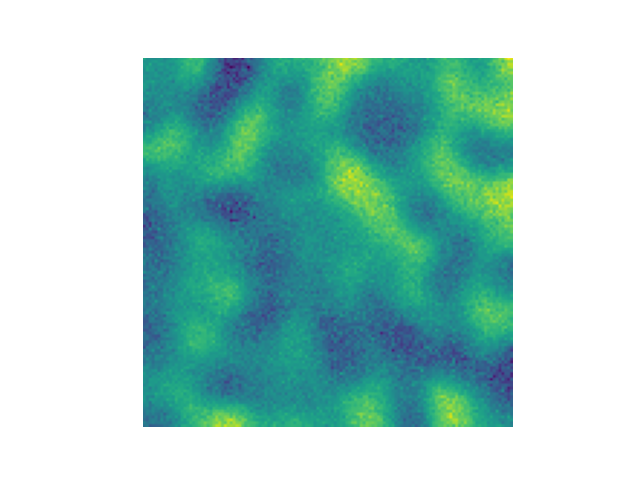}
    \caption{Measurement $v$}
\end{subfigure}
\begin{subfigure}[b]{0.23\linewidth}
    \includegraphics[width=\linewidth]{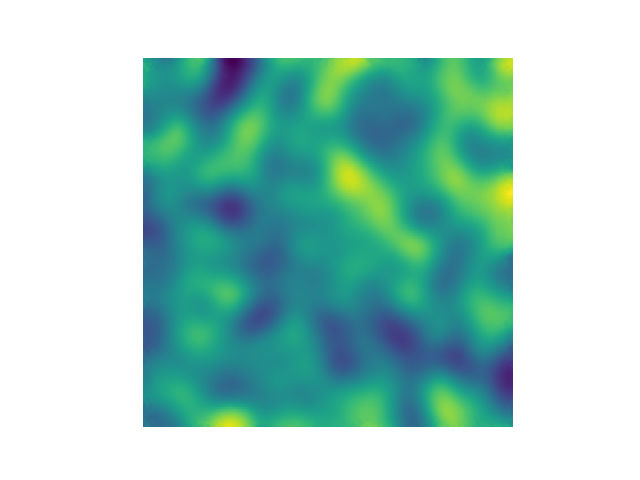}
    \caption{Posterior mean $v$}
\end{subfigure}
\begin{subfigure}[b]{0.23\linewidth}
    \includegraphics[width=\linewidth]{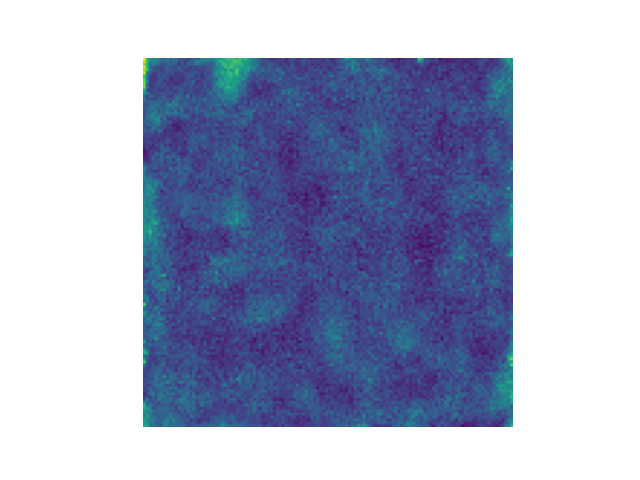}
    \caption{Posterior std. $v$}
\end{subfigure}

\caption{Ground truth, measurement, posterior mean and standard deviations of the model provided for the $u$- and $v$-component.}
\label{fig:pde_plots}
\end{figure}
\begin{figure}
    \centering
    \includegraphics[width=0.44\linewidth]{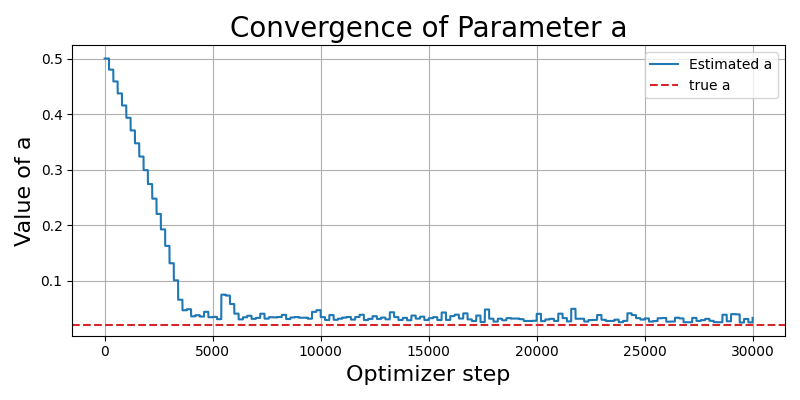}
    \includegraphics[width=0.44\linewidth]{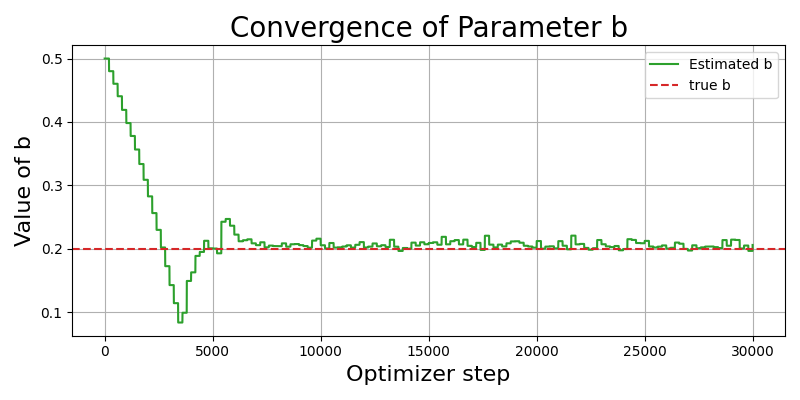}
    \caption{Convergence of a and b for the PDE example.}
    \label{fig:pde_convab}
\end{figure}
\section{Conclusions}
In this work, we conducted a detailed analysis of a mixed-noise error model frequently encountered in physical and chemical applications, and established a convergence result under idealized but interpretable conditions. Compared to the general conditions in \cite{em_conv}, we derived alternative criteria that are easier to verify and tailored to the structure of our noise model. In other words, we rephrased the conditions in a more Bayesian language.

Still, it is unclear whether our EM convergence theory applies to noise models with compact support, and whether it can be shown that our probabilistic FOS translates into EM convergence under the conditions on $\gamma$ and $c$.

Moreover, we extended the EM algorithm to high-dimensional inverse problems by incorporating the flow matching framework, enabling efficient posterior approximation and parameter estimation in imaging settings. We hope that our analysis contributes to a better theoretical understanding of mixed-noise models and supports the development of convergence guarantees in simulation-based inference methods \cite{cranmer_sbi, pmlr-v235-gloeckler24a}.

\section*{Acknowledgments}
PH and GS are grateful for 
 funding within the DFG-SPP 2298 ”Theoretical Foundations of Deep Learning” (Project: STE 571/17-1)
and
RG, CS, GS 
for the support from MATH+ project AA5-5, funded by DFG under Germany's Excellence Strategy – The Berlin Mathematics
Research Center MATH+ (EXC-2046/1, project ID: 390685689 ).

\bibliographystyle{siam}
\bibliography{biblio}

\end{document}